\documentclass[letterpaper]{article}
\usepackage{aaai18}
\usepackage{times}
\usepackage{helvet}
\usepackage{courier}
\usepackage{url}
\usepackage{graphicx}

\usepackage{amsmath}
\usepackage{mathrsfs}
\usepackage{amsfonts}
\usepackage{amsthm}
\usepackage{stmaryrd}
\usepackage{xcolor}
\usepackage{colortbl}
\usepackage{cases}
\usepackage{subcaption}
\usepackage{multirow}
\usepackage{chngcntr}
\usepackage{enumitem}
\usepackage{mathtools}
\usepackage{booktabs}
\usepackage{lipsum} 

\usepackage{trackchanges}

\usepackage{algorithm}
\usepackage{algorithmic}

\newcommand{\cX}{\mathcal{X}}

\newcommand{\cM}{\mathcal{M}}

\newcommand{\cA}{\mathcal{A}}

\newcommand{\cZ}{\mathcal{Z}}

\newcommand{\cT}{\mathcal{T}}

\newcommand{\cL}{\mathcal{L}}

\newcommand{\bR}{\mathbb{R}}

\DeclareMathOperator*{\expect}{{\huge \mathbb{E}}}

\newtheorem{lem}{Lemma}

\newtheorem{thm}{Theorem}

\DeclareMathOperator*{\argmin}{arg\,min}
\DeclareMathOperator*{\argmax}{arg\,max}

\newcommand{\eqnref}[1]{(\ref{eqn:#1})}

\def \cfo {\textsc{c51}}
\def \dqn {\textsc{dqn}}
\def \td {\textsc{td}($0$)}
\def \qrtd {\textsc{qrtd}}
\def \qrdqn {\textsc{qr-dqn}}

\def \cTpi {\cT^\pi}

\def \dip {\bar d_p}

\def \qcZ {\cZ_{Q}}

\usepackage{verbatim}
\usepackage{cancel}
\usepackage{thmtools,thm-restate}

\newcommand{\citet}[1]{\citeauthor{#1} (\citeyear{#1})}

\begin{document}
%

%
\title{Distributional Reinforcement Learning with Quantile Regression}
\author{Will Dabney\\ DeepMind \And Mark Rowland\\ University of Cambridge\thanks{Contributed during an internship at DeepMind.} \And Marc G. Bellemare\\ Google Brain \And R\'emi Munos\\ DeepMind}

\nocopyright

\maketitle
\begin{abstract}
In reinforcement learning an agent interacts with the environment by taking actions and observing the next state and reward. When sampled probabilistically, these state transitions, rewards, and actions can all induce randomness in the observed long-term return. Traditionally, reinforcement learning algorithms average over this randomness to estimate the value function. In this paper, we build on recent work advocating a distributional approach to reinforcement learning in which the distribution over returns is modeled explicitly instead of only estimating the mean. That is, we examine methods of learning the \textit{value distribution} instead of the value function. We give results that close a number of gaps between the theoretical and algorithmic results given by \citet{c51}. First, we extend existing results to the approximate distribution setting. Second, we present a novel distributional reinforcement learning algorithm consistent with our theoretical formulation. Finally, we evaluate this new algorithm on the Atari 2600 games, observing that it significantly outperforms many of the recent improvements on $\dqn$, including the related distributional algorithm $\cfo$.
\end{abstract}

\section{Introduction}

In reinforcement learning, the \emph{value} of an action $a$ in state $s$ describes the expected return, or discounted sum of rewards, obtained from beginning in that state, choosing action $a$, and subsequently following a prescribed policy. Because knowing this value for the \emph{optimal policy} is sufficient to act optimally, it is the object modelled by classic value-based methods such as SARSA \cite{rummery94online} and Q-Learning \cite{watkins1992q}, which use Bellman's equation \cite{bellman57dynamic} to efficiently reason about value.

Recently, \citet{c51} showed that the distribution of the random returns, whose expectation constitutes the aforementioned value, can be described by the distributional analogue of Bellman's equation, echoing previous results in risk-sensitive reinforcement learning \cite{heger1994consideration,morimura10parametric,chow2015risk}. In this previous work, however, the authors argued for the usefulness in modeling this \emph{value distribution} in and of itself. Their claim was asserted by exhibiting a distributional reinforcement learning algorithm, $\cfo$, which achieved state-of-the-art on the suite of benchmark Atari 2600 games \cite{bellemare13arcade}.

One of the theoretical contributions of the $\cfo$ work was a proof that the distributional Bellman operator is a contraction in a maximal form of the Wasserstein metric between probability distributions. In this context, the Wasserstein metric is particularly interesting because it does not suffer from disjoint-support issues \cite{wgan} which arise when performing Bellman updates. Unfortunately, this result does not directly lead to a practical algorithm: as noted by the authors, and further developed by \citet{bellemare17cramer}, the Wasserstein metric, viewed as a loss, cannot generally be minimized using stochastic gradient methods. 

This negative result left open the question as to whether it is possible to devise an online distributional reinforcement learning algorithm which takes advantage of the contraction result. Instead, the $\cfo$ algorithm first performs a heuristic projection step, followed by the minimization of a KL divergence between projected Bellman update and prediction. The work therefore leaves a theory-practice gap in our understanding of distributional reinforcement learning, which makes it difficult to explain the good performance of $\cfo$. Thus, the existence of a distributional algorithm that operates end-to-end on the Wasserstein metric remains an open question.

In this paper, we answer this question affirmatively. By appealing to the theory of quantile regression \cite{qrbook}, we show that there exists an algorithm, applicable in a stochastic approximation setting, which can perform distributional reinforcement learning over the Wasserstein metric. Our method relies on the following techniques:
\begin{itemize}
    \item We ``transpose'' the parametrization from $\cfo$: whereas the former uses $N$ fixed locations for its approximation distribution and adjusts their probabilities, we assign fixed, uniform probabilities to $N$ adjustable locations;
    \item We show that \emph{quantile regression} may be used to stochastically adjust the distributions' locations so as to minimize the Wasserstein distance to a target distribution.
    \item We formally prove contraction mapping results for our overall algorithm, and use these results to conclude that our method performs distributional RL end-to-end under the Wasserstein metric, as desired.
\end{itemize}

The main interest of the original distributional algorithm was its state-of-the-art performance, despite still acting by maximizing expectations. One might naturally expect that a direct minimization of the Wasserstein metric, rather than its heuristic approximation, may yield even better results. We derive the Q-Learning analogue for our method ($\qrdqn$), apply it to the same suite of Atari 2600 games, and find that it achieves even better performance. By using a smoothed version of quantile regression, \emph{Huber quantile regression}, we gain an impressive $33\%$ median score increment over the already state-of-the-art $\cfo$.

\section{Distributional RL}

We model the agent-environment interactions by a Markov decision process (MDP) ($\mathcal{X}, \mathcal{A}, R, P, \gamma$) \cite{puterman94markov}, with $\mathcal{X}$ and $\mathcal{A}$ the state and action spaces, $R$ the random variable reward function, $P(x' | x, a)$ the probability of transitioning from state $x$ to state $x'$ after taking action $a$, and $\gamma \in [0, 1)$ the discount factor. A policy $\pi(\cdot | x)$ maps each state $x \in \mathcal{X}$ to a distribution over $\mathcal{A}$.

For a fixed policy $\pi$, the \textit{return}, $Z^\pi = \sum_{t = 0}^\infty \gamma^t R_t$, is a random variable representing the sum of discounted rewards observed along one trajectory of states while following $\pi$. Standard RL algorithms estimate the expected value of $Z^\pi$, the \textit{value function}, 
\begin{equation}
    V^\pi(x) := \mathbb{E} \left[ Z^\pi(x) \right] = \mathbb{E} \left[ \sum_{t=0}^\infty \gamma^t R(x_t, a_t)\ |\ x_0 = x\right].
\end{equation}
Similarly, many RL algorithms estimate the action-value function,
\begin{eqnarray}
    Q^\pi(x, a) := \mathbb{E} \left[ Z^\pi(x, a) \right] = \mathbb{E} \left[ \sum_{t=0}^\infty \gamma^t R(x_t, a_t) \right],\\
    \nonumber x_t \sim P(\cdot | x_{t-1}, a_{t-1}), a_t \sim \pi(\cdot | x_t),x_0=x, a_0=a.
\end{eqnarray}

The $\epsilon$-greedy policy on $Q^\pi$ chooses actions uniformly at random with probability $\epsilon$ and otherwise according to $\argmax_a Q^\pi(x, a)$. 

In distributional RL the distribution over returns (i.e. the probability law of $Z^\pi$), plays the central role and replaces the value function. We will refer to the value distribution by its random variable. When we say that the value function is the mean of the value distribution we are saying that the value function is the expected value, taken over all sources of intrinsic randomness \cite{goldstein1981intrinsic}, of the value distribution. This should highlight that the value distribution is not designed to capture the uncertainty in the estimate of the value function \cite{dearden98bayesian,engel05reinforcement}, that is the \textit{parametric uncertainty}, but rather the randomness in the returns intrinsic to the MDP.

Temporal difference (TD) methods significantly speed up the learning process by incrementally improving an estimate of $Q^\pi$ using dynamic programming through the \textit{Bellman operator} \cite{bellman57dynamic},
\begin{equation}
    \nonumber \cTpi Q(x, a) = \mathbb{E} \left[ R(x, a) \right] + \gamma \mathbb{E}_{P, \pi} \left[ Q(x', a') \right].
\end{equation}
Similarly, the value distribution can be computed through dynamic programming using a \textit{distributional Bellman operator} \cite{c51},
\begin{eqnarray}
    \cTpi Z(x, a) :\overset{D}{=} R(x, a) + \gamma Z(x', a'),\\
    \nonumber x' \sim P(\cdot | x, a), a' \sim \pi(\cdot | x'),
\end{eqnarray}
where $Y :\overset{D}{=} U$ denotes equality of probability laws, that is the random variable $Y$ is distributed according to the same law as $U$.

\begin{figure}[t]
\begin{center}
\includegraphics[width=.4\textwidth]{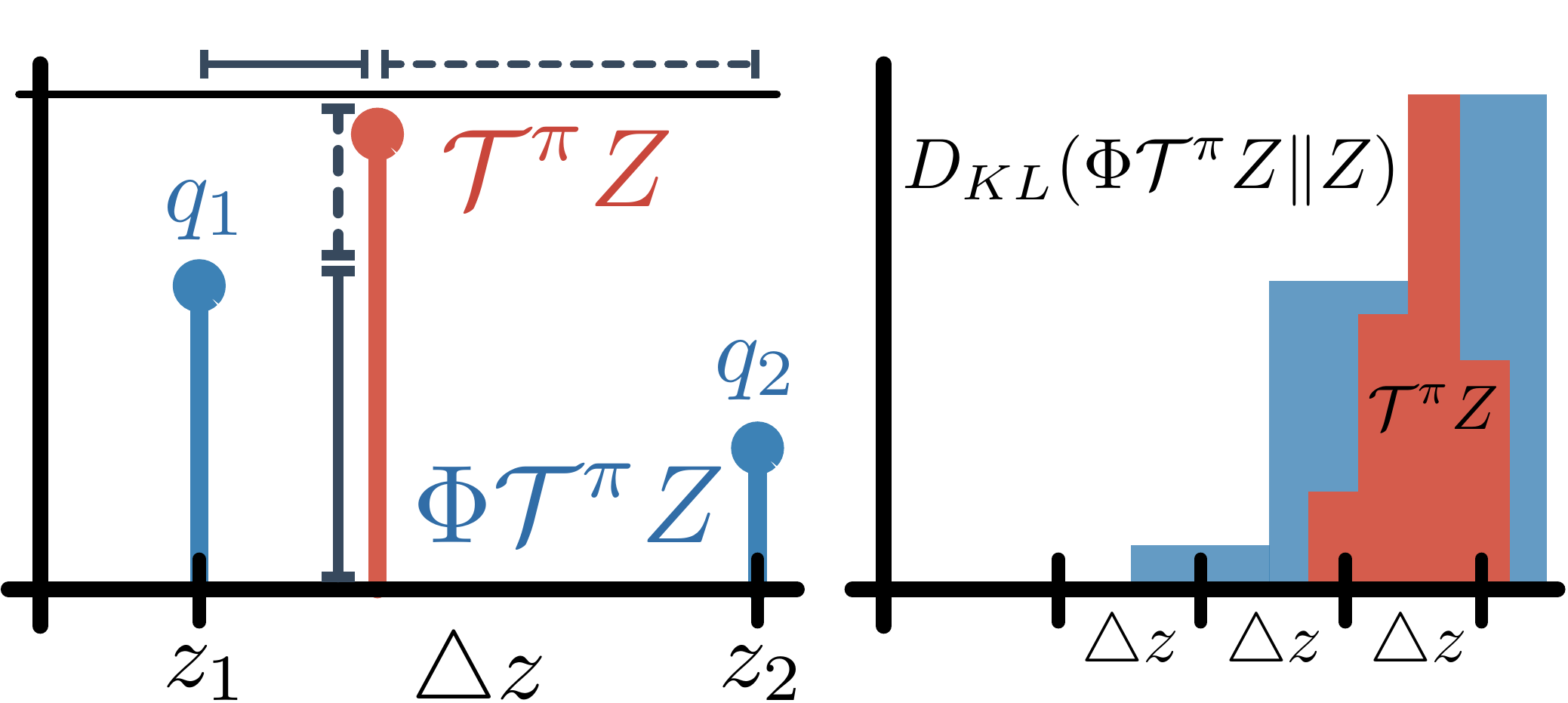}
\end{center}
\caption{Projection used by $\cfo$ assigns mass inversely proportional to distance from nearest support. Update minimizes KL between projected target and estimate.\label{fig:c51proj}}
\end{figure}

The $\cfo$ algorithm models $Z^\pi(x,a)$ using a discrete distribution supported on a ``comb'' of fixed locations $z_1 \leq \cdots \leq z_N$ uniformly spaced over a predetermined interval. The parameters of that distribution are the probabilities $q_i$, expressed as logits, associated with each location $z_i$. Given a current value distribution, the $\cfo$ algorithm applies a projection step $\Phi$ to map the target $\cTpi Z$ onto its finite element support, followed by a Kullback-Leibler (KL) minimization step (see Figure~\ref{fig:c51proj}). $\cfo$ achieved state-of-the-art performance on Atari 2600 games, but did so with a clear disconnect with the theoretical results of \citet{c51}. We now review these results before extending them to the case of approximate distributions.

\subsection{The Wasserstein Metric}

The $p$-Wasserstein metric $W_p$, for $p \in [1, \infty]$, also known as the Mallows metric \cite{bickel81asymptotic} or the Earth Mover's Distance (EMD) when $p=1$ \cite{levina2001earth}, is an integral probability metric between distributions.
The $p$-Wasserstein distance is characterized as the $L^p$ metric on inverse cumulative distribution functions (inverse CDFs) \cite{muller1997integral}. That is, the $p$-Wasserstein metric between distributions $U$ and $Y$ is given by,\footnote{For $p = \infty$, $W_\infty(Y, U) = \sup_{\omega \in [0, 1]} |F_Y^{-1}(\omega) - F_U^{-1}(\omega)|$.}
\begin{equation}
    W_p(U, Y) = \left( \int_0^1 |F_Y^{-1}(\omega) - F_U^{-1}(\omega)|^p d\omega \right)^{1/p} \, ,
\end{equation}
where for a random variable $Y$, the inverse CDF $F^{-1}_Y$ of $Y$ is defined by
\begin{equation}\label{eqn:quantile}
    F^{-1}_Y(\omega) := \inf \{ y \in \mathbb{R} : \omega \le F_Y(y) \} \, ,
\end{equation}
where $F_Y(y) = Pr(Y \leq y)$ is the CDF of $Y$. Figure~\ref{fig:wmin} illustrates the 1-Wasserstein distance as the area between two CDFs.

Recently, the Wasserstein metric has been the focus of increased research due to its appealing properties of respecting the underlying metric distances between outcomes \cite{wgan,bellemare17cramer}. 
Unlike the Kullback-Leibler divergence, the Wasserstein metric is a true probability metric and considers both the probability of and the distance between various outcome events. These properties make the Wasserstein well-suited to domains where an underlying similarity in outcome is more important than exactly matching likelihoods.

\subsection{Convergence of Distributional Bellman Operator}

In the context of distributional RL, let $\cZ$ be the space of action-value distributions with finite  moments:
\begin{align*}
    \mathcal{Z} = \{ &Z : \mathcal{X} \times \mathcal{A} \rightarrow \mathscr{P}(\mathbb{R}) |\\     &\mathbb{E}\left\lbrack |Z(x, a)|^p \right\rbrack < \infty, \ \forall (x, a), p \geq 1 \}.
\end{align*}
Then, for two action-value distributions $Z_1, Z_2 \in \cZ$, we will use the maximal form of the Wasserstein metric introduced by \cite{c51},
\begin{equation}\label{eqn:maxw}
    \dip(Z_1, Z_2) := \sup_{x,a} W_p(Z_1(x, a), Z_2(x, a)).
\end{equation}
It was shown that $\dip$ is a metric over value distributions. Furthermore, the distributional Bellman operator $\cTpi$ is a contraction in $\dip$, a result that we now recall.
\begin{lem}[Lemma 3, \citeauthor{c51} \citeyear{c51}]\label{lem:wasserstein_contraction_operator}
$\cTpi$ is a $\gamma$-contraction: for any two $Z_1, Z_2 \in \cZ$,
\begin{equation*}
    \dip(\cTpi Z_1, \cTpi Z_2) \le \gamma \dip(Z_1, Z_2) .
\end{equation*}
\end{lem}
Lemma \ref{lem:wasserstein_contraction_operator} tells us that $\dip$ is a useful metric for studying the behaviour of distributional reinforcement learning algorithms, in particular to show their convergence to the fixed point $Z^\pi$. Moreover, the lemma suggests that an effective way in practice to learn value distributions is to attempt to minimize the Wasserstein distance between a distribution $Z$ and its Bellman update $\mathcal{T}^\pi Z$, analogous to the way that TD-learning attempts to iteratively minimize the $L^2$ distance between $Q$ and $\mathcal{T} Q$. 

Unfortunately, another result shows that we cannot in general minimize the Wasserstein metric (viewed as a loss) using stochastic gradient descent.
\begin{thm}[Theorem 1, \citeauthor{bellemare17cramer} \citeyear{bellemare17cramer}]\label{thm:biased_gradients}
Let $\hat{Y}_m := \tfrac{1}{m} \sum_{i=1}^m \delta_{Y_i}$ be the empirical distribution derived from samples $Y_1, \dots, Y_m$ drawn from a Bernoulli distribution $B$. Let $B_\mu$ be a Bernoulli distribution parametrized by $\mu$, the probability of the variable taking the value $1$. Then the minimum of the expected sample loss is in general different from the minimum of the true Wasserstein loss; that is,
\begin{equation*}
    \argmin_\mu \expect_{Y_{1:m}} \big [ W_p(\hat Y_m, B_\mu) \big ] \ne \argmin_\mu W_p(B , B_\mu) .
\end{equation*}
\end{thm}
This issue becomes salient in a practical context, where the value distribution must be approximated. Crucially, the $\cfo$ algorithm is not guaranteed to minimize any $p$-Wasserstein metric. This gap between theory and practice in distributional RL is not restricted to $\cfo$. \citet{morimura10parametric} parameterize a value distribution with the mean and scale of a Gaussian or Laplace distribution, and minimize the KL divergence between the target $\cTpi Z$ and the prediction $Z$. They demonstrate that value distributions learned in this way are sufficient to perform risk-sensitive Q-Learning. However, any theoretical guarantees derived from their method can only be asymptotic; the Bellman operator is at best a non-expansion in KL divergence.

\begin{figure}[t]
\begin{center}
\includegraphics[width=.48\textwidth]{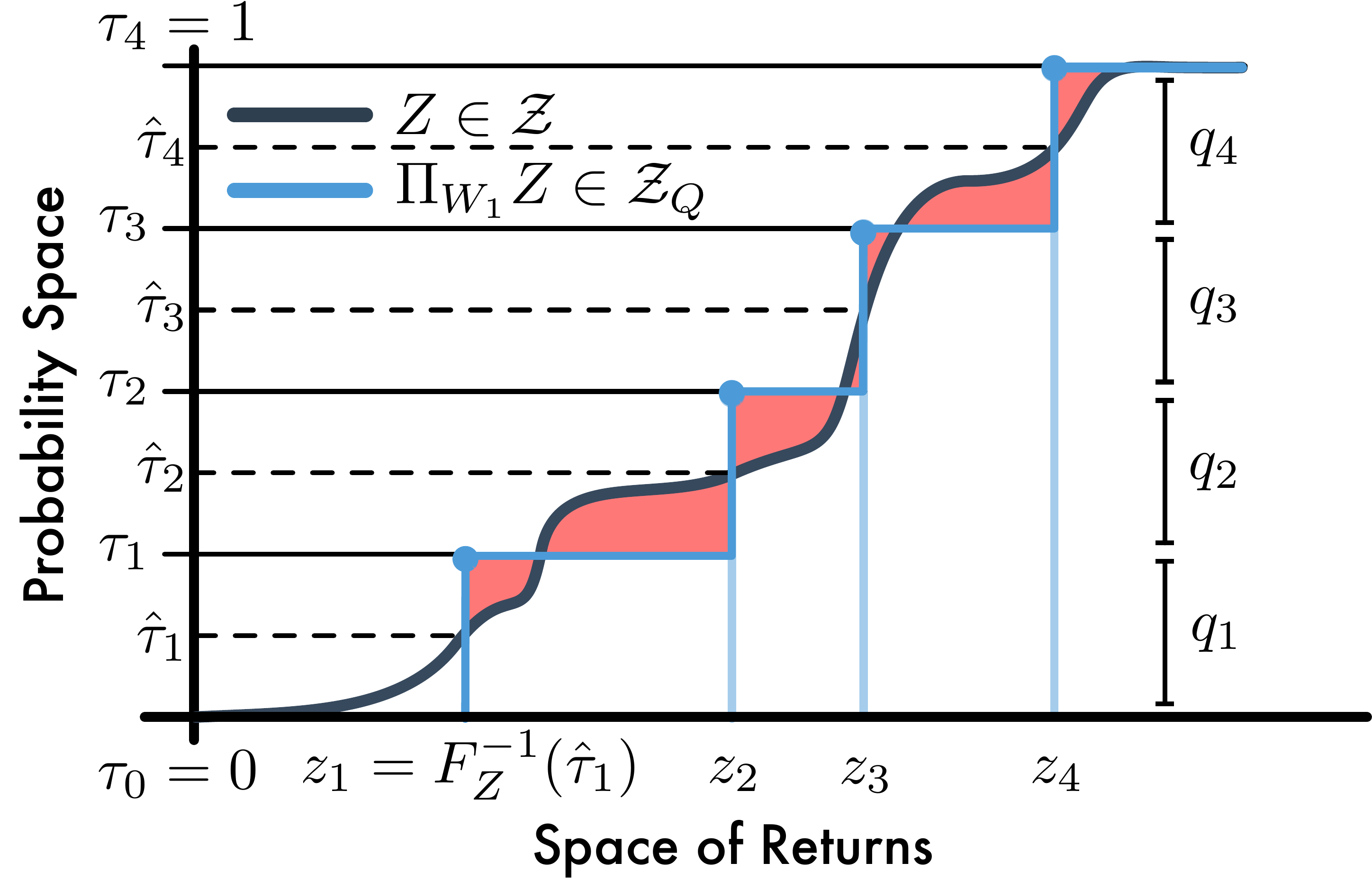}
\end{center}
\caption{1-Wasserstein minimizing projection onto $N=4$ uniformly weighted Diracs. Shaded regions sum to form the 1-Wasserstein error.\label{fig:wmin}}
\end{figure}

\section{Approximately Minimizing Wasserstein}
Recall that $\cfo$ approximates the distribution at each state by attaching variable (parametrized) probabilities $q_1, \dots, q_N$ to fixed locations $z_1 \leq \cdots \leq z_N$. Our approach is to ``transpose'' this parametrization by considering fixed probabilities but variable locations. 
Specifically, we take uniform weights, so that $q_i = 1/N$ for each $i=1,\ldots, N$.

Effectively, our new approximation aims to estimate \emph{quantiles} of the target distribution. Accordingly, we will call it a \emph{quantile distribution}, and let $\qcZ$ be the space of quantile distributions for fixed $N$. We will denote the cumulative probabilities associated with such a distribution (that is, the discrete values taken on by the CDF) by $\tau_1 ,\ldots, \tau_N$, so that $\tau_i = \frac{i}{N}$ for $i=1,\ldots,N$. We will also write $\tau_0=0$ to simplify notation.

Formally, let $\theta : \cX \times \cA \to \bR^N$ be some parametric model. A quantile distribution $Z_\theta \in \qcZ$ maps each state-action pair $(x,a)$ to a uniform probability distribution supported on $\{ \theta_i(x,a) \}$. That is, 
\begin{equation}\label{eqn:definition_quantile_distribution}
    Z_\theta(x,a) := \tfrac{1}{N} \sum_{i=1}^N \delta_{\theta_i(x,a)},
\end{equation}
where $\delta_z$ denotes a Dirac at $z \in \mathbb{R}$.

Compared to the original parametrization, the benefits of a parameterized quantile distribution are threefold. First, (1) we are not restricted to prespecified bounds on the support, or a uniform resolution, potentially leading to significantly more accurate predictions when the range of returns vary greatly across states. This also (2) lets us do away with the unwieldy projection step present in $\cfo$, as there are no issues of disjoint supports. Together, these obviate the need for domain knowledge about the bounds of the return distribution when applying the algorithm to new tasks. Finally, (3) this reparametrization allows us to minimize the Wasserstein loss, without suffering from biased gradients, specifically, using \emph{quantile regression}.

\subsection{The Quantile Approximation}

It is well-known that in reinforcement learning, the use of function approximation may result in instabilities in the learning process \cite{tsitsiklis97analysis}. Specifically, the Bellman update projected onto the approximation space may no longer be a contraction. In our case, we analyze the distributional Bellman update, projected onto a parameterized quantile distribution, and prove that the combined operator is a contraction.

\subsubsection{Quantile Projection}

We are interested in quantifying the projection of an arbitrary value distribution $Z \in \cZ$ onto $\qcZ$, that is
$$\Pi_{W_1} Z := \argmin_{Z_\theta \in \qcZ} W_1(Z, Z_\theta),$$

Let $Y$ be a distribution with bounded first moment and $U$ a uniform distribution over $N$ Diracs as in \eqnref{definition_quantile_distribution}, with support $\{ \theta_1, \dots, \theta_N \}$. Then
\begin{equation*}
    W_1(Y, U) = \sum_{i=1}^{N} \int_{\tau_{i-1}}^{\tau_{i}} | F_Y^{-1}(\omega) - \theta_i| d\omega.
\end{equation*}
\begin{restatable}{lem}{wonemidpoint}\label{w1_midpoint}
For any $\tau, \tau' \in [0, 1]$ with $\tau < \tau'$ and cumulative distribution function $F$ with inverse $F^{-1}$, the set of $\theta \in \mathbb{R}$ minimizing
\begin{equation*}
    \int_\tau^{\tau'} | F^{-1}(\omega) - \theta | d\omega \, ,
\end{equation*}
is given by
\begin{equation*}
    \left\{\theta \in \mathbb{R} \bigg| F(\theta) = \left(\frac{\tau + \tau'}{2}\right)\right\}.
\end{equation*}
In particular, if $F^{-1}$ is the inverse CDF, then $F^{-1}((\tau+\tau')/2)$ is always a valid minimizer, and if $F^{-1}$ is continuous at $(\tau+\tau')/2$, then $F^{-1}((\tau+\tau')/2)$ is the unique minimizer.
\end{restatable}

These \textit{quantile midpoints} will be denoted by $\hat{\tau}_i = \frac{\tau_{i-1} + \tau_{i}}{2}$ for $1 \le i \le N$.
Therefore, by Lemma~\ref{w1_midpoint}, the values for $\{\theta_1, \theta_1, \ldots, \theta_N\}$ that minimize $W_1(Y, U)$ are given by $\theta_i = F_Y^{-1}(\hat{\tau}_i)$. Figure~\ref{fig:wmin} shows an example of the quantile projection $\Pi_{W_1} Z$ minimizing the $1$-Wasserstein distance to $Z$.\footnote{We save proofs for the appendix due to space limitations.}

\subsection{Quantile Regression}

The original proof of Theorem \ref{thm:biased_gradients} only states the \emph{existence} of a distribution whose gradients are biased. As a result, we might hope that our quantile parametrization leads to unbiased gradients. Unfortunately, this is not true.
\begin{restatable}{prop}{biasedgradients}\label{prop:biased_transpose_gradients}
Let $Z_\theta$ be a quantile distribution, and $\hat{Z}_m$ the empirical distribution composed of $m$ samples from $Z$. Then for all $p \ge 1$, there exists a $Z$ such that
\begin{equation*}
    \argmin \expect [ W_p(\hat{Z}_m, Z_\theta) ] \neq \argmin W_p(Z, Z_\theta) .
\end{equation*}
\end{restatable}

However, there is a method, more widely used in economics than machine learning, for unbiased stochastic approximation of the quantile function. \textit{Quantile regression}, and \textit{conditional quantile regression}, are methods for approximating the quantile functions of distributions and conditional distributions respectively \cite{qrbook}. These methods have been used in a variety of settings where outcomes have intrinsic randomness \cite{koenker2001quantile}; from food expenditure as a function of household income \cite{engel1857productions}, to studying value-at-risk in economic models \cite{taylor1999quantile}.

The quantile regression loss, for quantile $\tau \in [0,1]$, is an asymmetric convex loss function that penalizes overestimation errors with weight $\tau$ and underestimation errors with weight $1 - \tau$. For a distribution $Z$, and a given quantile $\tau$, the value of the quantile function $F_Z^{-1}(\tau)$ may be characterized as the minimizer of the \emph{quantile regression loss}
\begin{align}
    \nonumber \cL^\tau_{\textsc{qr}}(\theta) := \mathbb{E}_{\hat{Z} \sim Z}[\rho_\tau(\hat{Z} - \theta)] \, , \text{ where}\\
    \rho_\tau(u) = u(\tau - \delta_{\{u < 0\}}),\ \forall u \in \mathbb{R}.\label{eqn:qr_loss}
\end{align}
More generally, by Lemma \ref{w1_midpoint} we have that the minimizing values of $\{\theta_1, \ldots, \theta_N\}$ for $W_1(Z, Z_\theta)$ are those that minimize the following objective:
\[
\sum_{i=1}^N \mathbb{E}_{\hat{Z} \sim Z}[ \rho_{\hat{\tau}_i}(\hat{Z} - \theta_i) ]
\]

In particular, this loss gives unbiased sample gradients. As a result, we can find the minimizing $\{ \theta_1, \ldots, \theta_N \}$ by stochastic gradient descent. 

\subsubsection{Quantile Huber Loss}
The quantile regression loss is not smooth at zero; as $u \to 0^+$, the gradient of Equation~\ref{eqn:qr_loss} stays constant. We hypothesized that this could limit performance when using non-linear function approximation. To this end, we also consider a modified quantile loss, called the \textit{quantile Huber loss}.\footnote{Our quantile Huber loss is related to, but distinct from that of \citet{aravkin2014sparse}.} This quantile regression loss acts as an asymmetric squared loss in an interval $[- \kappa, \kappa]$ around zero and reverts to a standard quantile loss outside this interval.

The Huber loss is given by \cite{huber1964robust},
\begin{align}
    \cL_\kappa(u) = \begin{cases}
        \frac{1}{2} u^2,\quad \ &\text{if } |u| \le \kappa\\
        \kappa(|u| - \frac{1}{2}\kappa),\quad \ &\text{otherwise}
    \end{cases}.
\end{align}
The quantile Huber loss is then simply the asymmetric variant of the Huber loss,
\begin{equation}\label{eqn:huber_quantile}
    \rho^\kappa_\tau(u) = |\tau - \delta_{\{ u < 0 \}}| \cL_\kappa(u).
\end{equation}
For notational simplicity we will denote $\rho^0_\tau = \rho_\tau$, that is, it will revert to the standard quantile regression loss.

\subsection{Combining Projection and Bellman Update}

We are now in a position to prove our main result, which states that the combination of the projection implied by quantile regression with the Bellman operator is a contraction. The result is in $\infty$-Wasserstein metric, i.e. the size of the largest gap between the two CDFs.
\begin{restatable}{prop}{Winftycontract}
Let $\Pi_{W_1}$ be the quantile projection defined as above, and when applied to value distributions gives the projection for each state-value distribution. For any two value distributions $Z_1, Z_2 \in \cZ$ for an MDP with countable state and action spaces,
\begin{equation}
    \bar{d}_\infty(\Pi_{W_1} \cTpi Z_1, \Pi_{W_1} \cTpi Z_2) \le \gamma \bar{d}_\infty (Z_1, Z_2).
\end{equation}
\end{restatable}
We therefore conclude that the combined operator $\Pi_{W_1} \cTpi$ has a unique fixed point $\hat{Z}^\pi$, and the repeated application of this operator, or its stochastic approximation, converges to $\hat{Z}^\pi$. Because ${\bar d}_p \le {\bar d}_\infty$, we conclude that convergence occurs for all $p\in [1,\infty]$. Interestingly, the contraction property does not directly hold for $p < \infty$; see Lemma \ref{lem:dpnocontraction} in the appendix.

\section{Distributional RL using Quantile Regression}
We can now form a complete algorithmic approach to distributional RL consistent with our theoretical results. That is, approximating the value distribution with a parameterized quantile distribution over the set of quantile midpoints, defined by Lemma~\ref{w1_midpoint}. Then, training the location parameters using quantile regression (Equation~\ref{eqn:qr_loss}).

\subsection{Quantile Regression Temporal Difference Learning}

Recall the standard TD update for evaluating a policy $\pi$,
\begin{eqnarray}
    \nonumber V(x) \leftarrow V(x) + \alpha (r + \gamma V(x') - V(x)),\\
    \nonumber a \sim \pi(\cdot | x), r \sim R(x, a), x' \sim P(\cdot | x, a).
\end{eqnarray}
TD allows us to update the estimated value function with a single unbiased sample following $\pi$. Quantile regression also allows us to improve the estimate of the quantile function for some target distribution, $Y(x)$, by observing samples $y \sim Y(x)$ and minimizing Equation~\ref{eqn:qr_loss}.

Furthermore, we have shown that by estimating the quantile function for well-chosen values of $\tau \in (0, 1)$ we can obtain an approximation with minimal 1-Wasserstein distance from the original (Lemma~\ref{w1_midpoint}). Finally, we can combine this with the distributional Bellman operator to give a target distribution for quantile regression. This gives us the quantile regression temporal difference learning ($\qrtd$) algorithm, summarized simply by the update,
\begin{eqnarray}
    &\theta_i(x) \leftarrow \theta_i(x) + \alpha (\hat{\tau}_i - \delta_{\{r + \gamma z' < \theta_i(x))\} }),\\
    \nonumber &a \sim \pi(\cdot | x), r \sim R(x, a), x' \sim P(\cdot | x, a), z' \sim Z_\theta(x'),
\end{eqnarray}
where $Z_\theta$ is a quantile distribution as in \eqnref{definition_quantile_distribution}, and $\theta_i(x)$ is the estimated value of $F^{-1}_{Z^\pi(x)}(\hat \tau_i)$ in state $x$. It is important to note that this update is for each value of $\hat \tau_i$ and is defined for a single sample from the next state value distribution. In general it is better to draw many samples of $z' \sim Z(x')$ and minimize the expected update. A natural approach in this case, which we use in practice, is to compute the update for all pairs of ($\theta_i(x), \theta_j(x')$). Next, we turn to a control algorithm and the use of non-linear function approximation.

\subsection{Quantile Regression $\dqn$}

Q-Learning is an off-policy reinforcement learning algorithm built around directly learning the optimal action-value function using the Bellman optimality operator \cite{watkins1992q},
\begin{equation*}
    \cT Q(x, a) = \expect \left[ R(x, a) \right] + \gamma \expect_{x' \sim P} \left[ \max_{a'} Q(x', a') \right].
\end{equation*}

The distributional variant of this is to estimate a state-action value distribution and apply a distributional Bellman optimality operator,
\begin{eqnarray}
    &\cT Z(x, a) = R(x, a) + \gamma Z(x', a'),\\
    \nonumber &x' \sim P(\cdot|x, a), a' = \argmax_{a'} \expect_{z \sim Z(x', a')} \left[ z \right].
\end{eqnarray}
Notice in particular that the action used for the next state is the greedy action with respect to the mean of the next state-action value distribution.

For a concrete algorithm we will build on the $\dqn$ architecture \cite{mnih15nature}. We focus on the minimal changes necessary to form a distributional version of $\dqn$. Specifically, we require three modifications to $\dqn$. First, we use a nearly identical neural network architecture as $\dqn$, only changing the output layer to be of size $|\mathcal{A}| \times N$, where $N$ is a hyper-parameter giving the number of quantile targets. Second, we replace the Huber loss used by $\dqn$\footnote{$\dqn$ uses gradient clipping of the squared error that makes it equivalent to a Huber loss with $\kappa = 1$.}, $L_{\kappa}(r_t + \gamma \max_{a'} Q(x_{t+1}, a') - Q(x_t, a_t))$ with $\kappa = 1$, with a quantile Huber loss (full loss given by Algorithm~\ref{alg:wdqn}). Finally, we replace RMSProp \cite{tieleman2012lecture} with Adam \cite{kingma2014adam}. We call this new algorithm quantile regression $\dqn$ ($\qrdqn$).

\begin{algorithm}[ht]
\caption{Quantile Regression Q-Learning}\label{alg:wdqn}
\begin{algorithmic}
\REQUIRE $N, \kappa$
\INPUT $x, a, r, x'$, $\gamma \in [0, 1)$
\STATE \textcolor{gray}{\# Compute distributional Bellman target}
\STATE $Q(x', a') := \sum\nolimits_j q_j \theta_j(x', a')$
\STATE $a^* \leftarrow \argmax_{a'} Q(x, a')$
\STATE $\cT \theta_j \leftarrow r + \gamma \theta_j(x', a^*),\ \quad \forall j$
\STATE \textcolor{gray}{\# Compute quantile regression loss (Equation~\ref{eqn:huber_quantile})}
\OUTPUT $\sum_{i=1}^{N} \expect_j \left[ \rho^\kappa_{\hat \tau_i}( \cT \theta_j - \theta_i(x, a)) \right]$
\end{algorithmic}
\end{algorithm}

Unlike $\cfo$, $\qrdqn$ does not require projection onto the approximating distribution's support, instead it is able to expand or contract the values arbitrarily to cover the true range of return values. As an additional advantage, this means that $\qrdqn$ does not require the additional hyper-parameter giving the bounds of the support required by $\cfo$. The only additional hyper-parameter of $\qrdqn$ not shared by $\dqn$ is the number of quantiles $N$, which controls with what resolution we approximate the value distribution. As we increase $N$, $\qrdqn$ goes from $\dqn$ to increasingly able to estimate the upper and lower quantiles of the value distribution. It becomes increasingly capable of distinguishing low probability events at either end of the cumulative distribution over returns. 

\begin{figure*}[ht]
\begin{center}
\includegraphics[width=\textwidth]{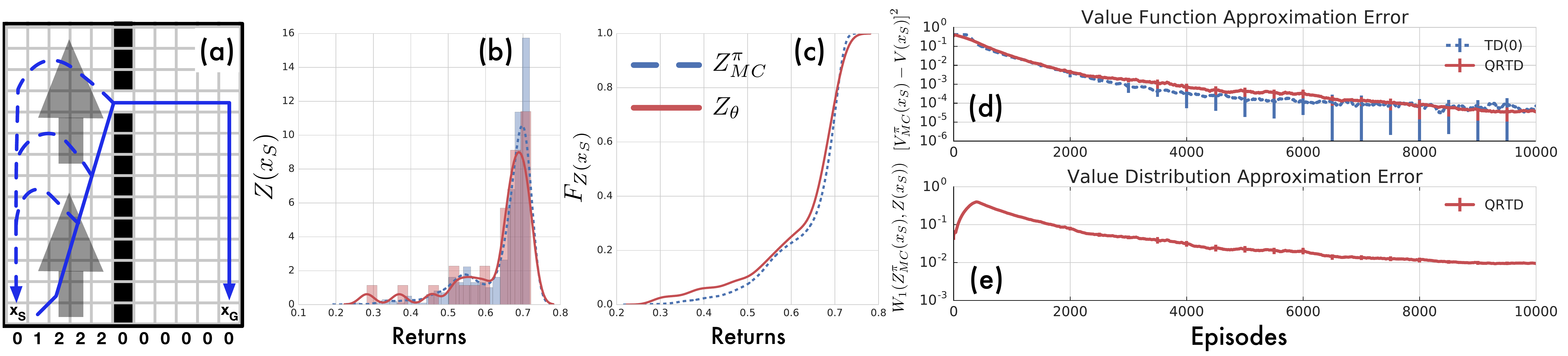}
\end{center}
\caption{(a) Two-room windy gridworld, with wind magnitude shown along bottom row. Policy trajectory shown by blue path, with additional cycles caused by randomness shown by dashed line. (b, c) (Cumulative) Value distribution at start state $x_S$, estimated by MC, $Z^\pi_{MC}$, and by $\qrtd$, $Z_\theta$. (d, e) Value function (distribution) approximation errors for $\td$ and $\qrtd$. \label{fig:windy_gw}}
\end{figure*}

\section{Experimental Results}

In the introduction we claimed that learning the distribution over returns had distinct advantages over learning the value function alone. We have now given theoretically justified algorithms for performing distributional reinforcement learning, $\qrtd$ for policy evaluation and $\qrdqn$ for control. In this section we will empirically validate that the proposed distributional reinforcement learning algorithms: (1) learn the true distribution over returns, (2) show increased robustness during training, and (3) significantly improve sample complexity and final performance over baseline algorithms.

\subsubsection{Value Distribution Approximation Error}

We begin our experimental results by demonstrating that $\qrtd$ actually learns an approximate value distribution that minimizes the $1$-Wasserstein to the ground truth distribution over returns. Although our theoretical results already establish convergence of the former to the latter, the empirical performance helps to round out our understanding.

We use a variant of the classic windy gridworld domain \cite{sutton98reinforcement}, modified to have two rooms and randomness in the transitions. Figure~\ref{fig:windy_gw}(a) shows our version of the domain, where we have combined the transition stochasticity, wind, and the doorway to produce a multi-modal distribution over returns when anywhere in the first room. Each state transition has probability $0.1$ of moving in a random direction, otherwise the transition is affected by wind moving the agent northward. The reward function is zero until reaching the goal state $x_G$, which terminates the episode and gives a reward of $1.0$. The discount factor is $\gamma = 0.99$.

We compute the ground truth value distribution for optimal policy $\pi$, learned by policy iteration, at each state by performing $1K$ Monte-Carlo (MC) rollouts and recording the observed returns as an empirical distribution, shown in Figure~\ref{fig:windy_gw}(b). Next, we ran both $\td$ and $\qrtd$ while following $\pi$ for $10K$ episodes. Each episode begins in the designated start state ($x_S$). Both algorithms started with a learning rate of $\alpha = 0.1$. For $\qrtd$ we used $N = 32$ and drop $\alpha$ by half every $2K$ episodes.

Let $Z^\pi_{MC}(x_S)$ be the MC estimated distribution over returns from the start state $x_S$, similarly $V^\pi_{MC}(x_S)$ its mean. In Figure~\ref{fig:windy_gw} we show the approximation errors at $x_S$ for both algorithms with respect to the number of episodes. In (d) we evaluated, for both $\td$ and $\qrtd$, the squared error, $(V^\pi_{MC} - V(x_S))^2$, and in (e) we show the $1$-Wasserstein metric for $\qrtd$, $W_1(Z^\pi_{MC}(x_S), Z(x_S))$, where $V(x_S)$ and $Z(x_S)$ are the expected returns and value distribution at state $x_S$ estimated by the algorithm. As expected both algorithms converge correctly in mean, and $\qrtd$ minimizes the $1$-Wasserstein distance to $Z^\pi_{MC}$.

\subsection{Evaluation on Atari 2600}

\begin{figure*}[ht]
\begin{center}
\includegraphics[width=\textwidth]{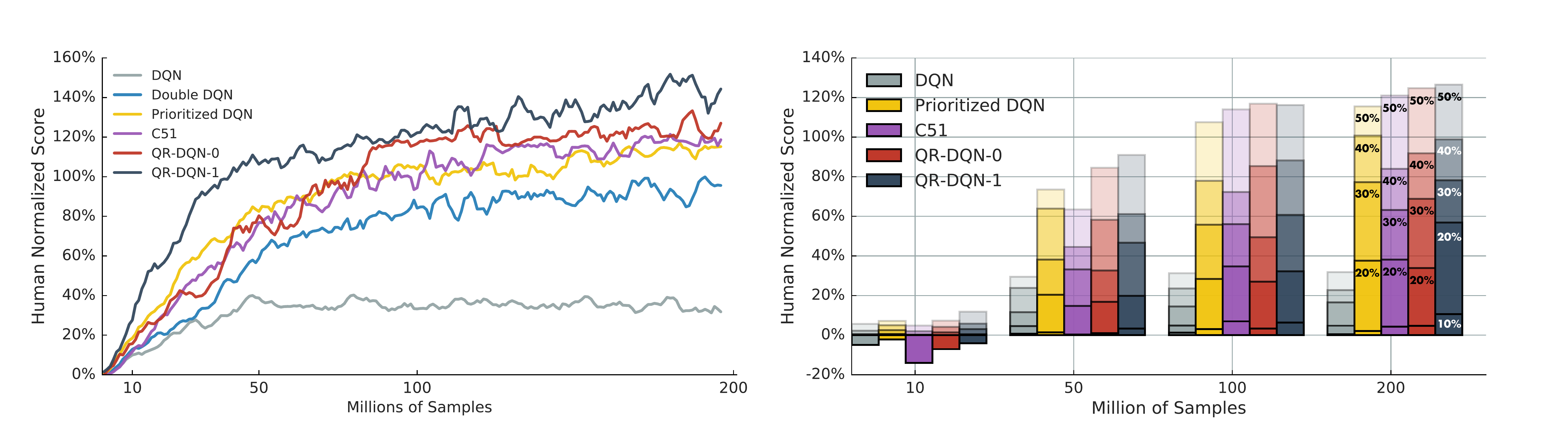}
\end{center}
\caption{Online evaluation results, in human-normalized scores, over 57 Atari 2600 games for 200 million training samples. (Left) Testing performance for one seed, showing median over games. (Right) Training performance, averaged over three seeds, showing percentiles (10, 20, 30, 40, and 50) over games.\label{fig:wdqn_test}}
\end{figure*}

We now provide experimental results that demonstrate the practical advantages of minimizing the Wasserstein metric end-to-end, in contrast to the $\cfo$ approach. We use the 57 Atari 2600 games from the Arcade Learning Environment (ALE) \cite{bellemare13arcade}. Both $\cfo$ and $\qrdqn$ build on the standard $\dqn$ architecture, and we expect both to benefit from recent improvements to $\dqn$ such as the dueling architectures \cite{wang2016dueling} and prioritized replay \cite{schaul16prioritized}. However, in our evaluations we compare the pure versions of $\cfo$ and $\qrdqn$ without these additions. We present results for both a strict quantile loss, $\kappa = 0$ ($\qrdqn$-$0$), and with a Huber quantile loss with $\kappa=1$ ($\qrdqn$-$1$). 

We performed hyper-parameter tuning over a set of five training games and evaluated on the full set of 57 games using these best settings ($\alpha = 0.00005$, $\epsilon_{ADAM} = 0.01/32$, and $N = 200$).\footnote{We swept over $\alpha$ in ($10^{-3}, 5\times 10^{-4}, 10^{-4}, 5 \times 10^{-5}, 10^{-5}$); $\epsilon_{ADAM}$ in ($0.01/32, 0.005/32, 0.001/32$); $N$ ($10, 50, 100, 200$)} As with $\dqn$ we use a target network when computing the distributional Bellman update. We also allow $\epsilon$ to decay at the same rate as in $\dqn$, but to a lower value of $0.01$, as is common in recent work \cite{c51,wang2016dueling,vanhasselt16deep}.

Out training procedure follows that of \citet{mnih15nature}'s, and we present results under two evaluation protocols: \textit{best agent} performance and \textit{online} performance. In both evaluation protocols we consider performance over all 57 Atari 2600 games, and transform raw scores into \textit{human-normalized scores} \cite{vanhasselt16deep}.

\begin{table}[ht]
\begin{center}
\begin{tabular}{ l | r | r | r | r }
\multicolumn{1}{c}{} & \mbox{\textbf{Mean}} & \mbox{\textbf{Median}} & $>$\mbox{\textbf{human}} & $>$\mbox{\textbf{DQN}} \\
\hline
\textsc{dqn}  &   228\% & 79\% & 24 & 0 \\
\textsc{ddqn}   &   307\% & 118\% & 33 & 43 \\
\textsc{Duel.}   &   373\% & 151\% & 37 & 50 \\
\textsc{Prior.}   &   434\% & 124\% & 39 & 48 \\
\textsc{Pr. Duel.}   &   592\% & 172\% & 39 & 44 \\
\hline
\hline
$\cfo$   &   701\% & 178\% & 40 & 50 \\
\hline
$\qrdqn$-$0$   &   881\% & 199\% & 38 & 52 \\
$\qrdqn$-$1$   &   \textbf{\textcolor{blue}{915\%}} & \textbf{\textcolor{blue}{211\%}} & \textbf{\textcolor{blue}{41}} & \textbf{\textcolor{blue}{54}} \\
\end{tabular}
\end{center}
\caption{Mean and median of \textit{best} scores across 57 Atari 2600 games, measured as percentages of human baseline \cite{nair15massively}.}
\label{fig:perc_scores}
\end{table}

\subsubsection{Best agent performance}

To provide comparable results with existing work we report test evaluation results under the best agent protocol. Every one million training frames, learning is frozen and the agent is evaluated for $500K$ frames while recording the average return. Evaluation episodes begin with up to $30$ random no-ops \cite{mnih15nature}, and the agent uses a lower exploration rate ($\epsilon = 0.001$). As training progresses we keep track of the best agent performance achieved thus far.

Table~\ref{fig:perc_scores} gives the best agent performance, at $200$ million frames trained, for $\qrdqn$, $\cfo$, $\dqn$, Double $\dqn$ \cite{vanhasselt16deep}, Prioritized replay \cite{schaul16prioritized}, and Dueling architecture \cite{wang2016dueling}. We see that $\qrdqn$ outperforms all previous agents in mean and median human-normalized score. 

\subsubsection{Online performance}
In this evaluation protocol (Figure~\ref{fig:wdqn_test}) we track the average return attained during each testing (left) and training (right) iteration. For the testing performance we use a single seed for each algorithm, but show online performance with no form of early stopping. For training performance, values are averages over three seeds. Instead of reporting only median performance, we look at the distribution of human-normalized scores over the full set of games. Each bar represents the score distribution at a fixed percentile ($10$th, $20$th, $30$th, $40$th, and $50$th). The upper percentiles show a similar trend but are omitted here for visual clarity, as their scale dwarfs the more informative lower half.

From this, we can infer a few interesting results. (1) Early in learning, most algorithms perform worse than random for at least $10\%$ of games. (2) $\qrtd$ gives similar improvements to sample complexity as prioritized replay, while also improving final performance. (3) Even at $200$ million frames, there are $10\%$ of games where all algorithms reach less than $10\%$ of human. This final point in particular shows us that all of our recent advances continue to be severely limited on a small subset of the Atari 2600 games.

\section{Conclusions}
The importance of the distribution over returns in reinforcement learning has been (re)discovered and highlighted many times by now. In \citet{c51} the idea was taken a step further, and argued to be a central part of approximate reinforcement learning. However, the paper left open the question of whether there exists an algorithm which could bridge the gap between Wasserstein-metric theory and practical concerns.

In this paper we have closed this gap with both theoretical contributions and a new algorithm which achieves state-of-the-art performance in Atari 2600.
There remain many promising directions for future work. Most exciting will be to expand on the promise of a richer policy class, made possible by action-value distributions. We have mentioned a few examples of such policies, often used for risk-sensitive decision making. However, there are many more possible decision policies that consider the action-value distributions as a whole.

Additionally, $\qrdqn$ is likely to benefit from the improvements on $\dqn$ made in recent years. For instance, due to the similarity in loss functions and Bellman operators we might expect that $\qrdqn$ suffers from similar over-estimation biases to those that Double $\dqn$ was designed to address \cite{vanhasselt16deep}. A natural next step would be to combine $\qrdqn$ with the non-distributional methods found in Table~\ref{fig:perc_scores}.

\section*{Acknowledgements}
The authors acknowledge the vital contributions of their colleagues at DeepMind. Special thanks to Tom Schaul, Audrunas Gruslys, Charles Blundell, and Benigno Uria for their early suggestions and discussions on the topic of quantile regression. Additionally, we are grateful for feedback from David Silver, Yee Whye Teh, Georg Ostrovski, Joseph Modayil, Matt Hoffman, Hado van Hasselt, Ian Osband, Mohammad Azar, Tom Stepleton, Olivier Pietquin, Bilal Piot; and a second acknowledgement in particular of Tom Schaul for his detailed review of an previous draft.


\bibliographystyle{aaai}

\newpage

\section*{Appendix}

\subsection{Proofs}

\wonemidpoint*

\begin{proof}
For any $\omega \in [0,1]$, the function $\theta \mapsto |F^{-1}(\omega) - \theta|$ is convex, and has subgradient given by
\[
\theta \mapsto \begin{cases}
1 & \text{\ if\ } \theta < F^{-1}(\omega) \\
[-1, 1] & \text{\ if\ } \theta = F^{-1}(\omega) \\
-1 & \text{\ if\ } \theta > F^{-1}(\omega) \, ,
\end{cases}
\]
so the function $\theta \mapsto \int_{\tau}^{\tau'} |F^{-1}(\omega) - \theta|d\omega$ is also convex, and has subgradient given by
\[
\theta \mapsto \int_\tau^{F(\theta)} -1  d\omega + \int_{F(\theta)}^{\tau'} 1 d\omega \, .
\]
Setting this subgradient equal to $0$ yields
\begin{align}\label{eq:subgradeqn}
(\tau + \tau') - 2F(\theta) = 0 \, ,
\end{align}
and since $F \circ F^{-1}$ is the identity map on $[0,1]$, it is clear that $\theta = F^{-1}((\tau + \tau')/2)$ satisfies Equation \ref{eq:subgradeqn}. Note that in fact any $\theta$ such that $F(\theta) = (\tau + \tau')/2$ yields a subgradient of $0$, which leads to a multitude of minimizers if $F^{-1}$ is not continuous at $(\tau + \tau')/2$.
\end{proof}

\biasedgradients*

\begin{proof}
Write $Z_\theta = \sum_{i=1}^N \frac{1}{N} \delta_{\theta_i}$, with $\theta_1 \leq \cdots \leq \theta_N$. We take $Z$ to be of the same form as $Z_\theta$. Specifically, consider $Z$ given by
\[
Z = \sum_{i=1}^N \frac{1}{N} \delta_i \, ,
\]
supported on the set $\{1,\ldots, N\}$,
and take $m=N$.
Then clearly the unique minimizing $Z_\theta$ for $W_p(Z, Z_\theta)$ is given by taking $Z_\theta = Z$. However, consider the gradient with respect to $\theta_1$ for the objective
\[
\mathbb{E}[W_p(\hat{Z}_N, Z_\theta)] \, .
\]
We have
\[
\nabla_{\theta_1} \mathbb{E}[W_p(\hat{Z}_N, Z_\theta)] |_{\theta_1 = 1}  = \mathbb{E}[\nabla_{\theta_1}W_p(\hat{Z}_N, Z_\theta)|_{\theta_1=1}] \, .
\]
In the event that the sample distribution $\hat{Z}_N$ has an atom at $1$, then the optimal transport plan pairs the atom of $Z_\theta$ at $\theta_1=1$ with this atom of $\hat{Z}_N$, and gradient with respect to $\theta_1$ of $W_p(\hat{Z}_N, Z_\theta)$ is $0$. If the sample distribution $\hat{Z}_N$ does not contain an atom at $1$, then the left-most atom of $\hat{Z}_N$ is greater than $1$ (since $Z$ is supported on $\{1,\ldots, N\}$. In this case, the gradient on $\theta_1$ is negative. Since this happens with non-zero probability, we conclude that 
\[
\nabla_{\theta_1} \mathbb{E}[W_p(\hat{Z}_N, Z_\theta)] |_{\theta_1 = 1} < 0 \, ,
\]
and therefore $Z_\theta = Z$ cannot be the minimizer of $\mathbb{E}[W_p(\hat{Z}_N, Z_\theta)]$.
\end{proof}

\Winftycontract*

\begin{proof}
We assume that instantaneous rewards given a state-action pair are deterministic; the general case is a straightforward generalization. Further, since the operator $\mathcal{T}^\pi$ is a $\gamma$-contraction in $\overline{d}_\infty$, it is sufficient to prove the claim in the case $\gamma = 1$. In addition, since Wasserstein distances are invariant under translation of the support of distributions, it is sufficient to deal with the case where $r(x, a) \equiv 0$ for all $(x, a) \in \mathcal{X} \times \mathcal{A}$. The proof then proceeds by first reducing to the case where every value distribution consists only of single Diracs, and then dealing with this reduced case using Lemma \ref{lem:1DiracCase}.

We write $Z(x, a) = \sum_{k=1}^N \frac{1}{N} \delta_{\theta_k(x, a)}$ and $Y(x, a) = \sum_{k=1}^N \frac{1}{N} \delta_{\psi_k(x, a)}$, for some functions $\theta, \psi : \mathcal{X} \times \mathcal{A} \rightarrow \mathbb{R}^n$. Let $(x ,a)$ be a state-action pair, and let $((x_i, a_i))_{i \in I}$ be all the state-action pairs that are accessible from $(x^\prime, a^\prime)$ in a single transition, where $I$ is a (finite or countable) indexing set. Write $p_i$ for the probability of transitioning from $(x^\prime, a^\prime)$ to $(x_i, a_i)$, for each $i \in I$. We now construct a new MDP and new value distributions for this MDP in which all distributions are given by single Diracs, with a view to applying Lemma \ref{lem:1DiracCase}. The new MDP is of the following form. We take the state-action pair $(x^\prime, a^\prime)$, and define new states, actions, transitions, and a policy $\widetilde{\pi}$, so that the state-action pairs accessible from $(x^\prime, a^\prime)$ in this new MDP are given by $((\widetilde{x}_i^j, \widetilde{a}_i^j)_{i\in I})_{j=1}^N$, and the probability of reaching the state-action pair $(\widetilde{x}_i^j, \widetilde{a}_i^j)$ is $p_i/n$. Further, we define new value distributions $\widetilde{Z}, \widetilde{Y}$ as follows. For each $i \in I$ and $j=1,\ldots,N$, we set:
\begin{align*}
\widetilde{Z}(\widetilde{x}_i^j, \widetilde{a}_i^j) = \delta_{\theta_j(x_i, a_i)} \\
\widetilde{Y}(\widetilde{x}_i^j, \widetilde{a}_i^j) = \delta_{\psi_j(x_i, a_i)} \, .
\end{align*}
The construction is illustrated in Figure \ref{fig:TransformedMDP}.
\begin{figure}
    \centering
    \includegraphics[keepaspectratio,width=0.36\textwidth]{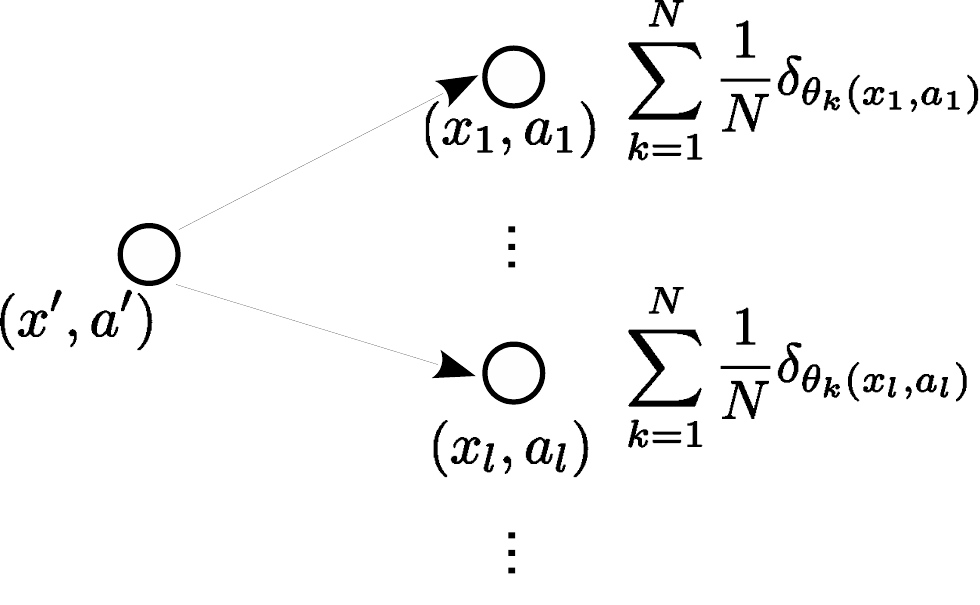}
    \vspace{0.25cm}
    \rule{0.47\textwidth}{0.4pt}
    
    \includegraphics[keepaspectratio,width=0.36\textwidth]{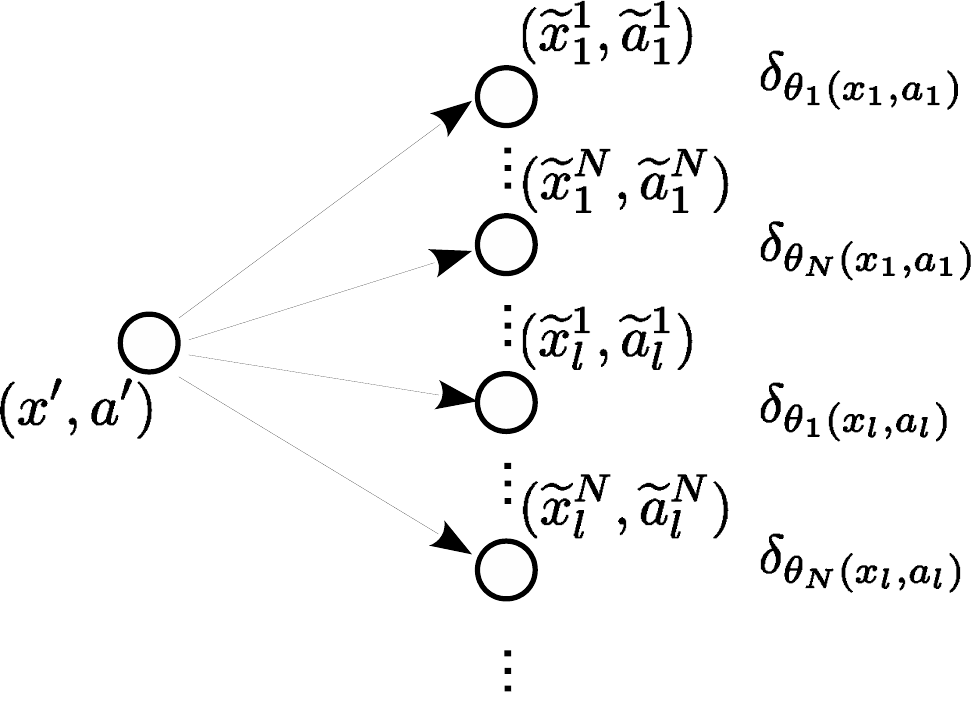}
    \caption{Initial MDP and value distribution $Z$ (top), and transformed MDP and value distribution $\widetilde{Z}$ (bottom).}
    \label{fig:TransformedMDP}
\end{figure}

Since, by Lemma \ref{lem:WinftyIsMaxQuantileDiff}, the $d_\infty$ distance between the 1-Wasserstein projections of two real-valued distributions is the max over the difference of a certain set of quantiles, we may appeal to Lemma \ref{lem:1DiracCase} to obtain the following:
\begin{align}
 & d_\infty(\Pi_{W_1} (\mathcal{T}^{\widetilde{\pi}} \widetilde{Z})(x^\prime, a^\prime), \Pi_{W_1}(\mathcal{T}^{\widetilde{\pi}} \widetilde{Y})(x^\prime, a^\prime) ) \nonumber  \\
 \leq & \sup_{\substack{i=1 \in I\\ j=1,\ldots,N } } |\theta_j(x_i, a_i) - \psi_j(x_i, a_i)| \nonumber\\
 = & \sup_{i=1 \in I} d_\infty(Z(x_i, a_i), Y(x_i, a_i)) \label{eq:ndiracsto1dirac}
\end{align}

Now note that by construction, $(\mathcal{T}^{\widetilde{\pi}} \widetilde{Z})(x^\prime, a^\prime)$ (respectively, $(\mathcal{T}^{\widetilde{\pi}} \widetilde{Y})(x^\prime, a^\prime)$) has the same distribution as $(\mathcal{T}^\pi Z)(x^\prime, a^\prime)$ (respectively, $(\mathcal{T}^\pi Y)(x^\prime, a^\prime)$), and so
\begin{align*}
 d_\infty(\Pi_{W_1} (\mathcal{T}^{\widetilde{\pi}} \widetilde{Z})(x^\prime, a^\prime), \Pi_{W_1}(\mathcal{T}^{\widetilde{\pi}} \widetilde{Y})(x^\prime, a^\prime) ) \\
= d_\infty(\Pi_{W_1} (\mathcal{T}^{\pi} Z)(x^\prime, a^\prime), \Pi_{W_1}(\mathcal{T}^{\pi} Y)(x^\prime, a^\prime) ) \, .
\end{align*}
Therefore, substituting this into the Inequality \ref{eq:ndiracsto1dirac}, we obtain
\begin{align*}
    &d_\infty(\Pi_{W_1} (\mathcal{T}^{\pi} Z)(x^\prime, a^\prime), \Pi_{W_1}(\mathcal{T}^{\pi} Y)(x^\prime, a^\prime) ) \\
    \leq&  \sup_{i \in I} d_\infty(Z(x_i, a_i), Y(x_i, a_i)) \, .
\end{align*}
Taking suprema over the initial state $(x^\prime, a^\prime)$ then yields the result.
\end{proof}

\subsection{Supporting results}

\begin{lem}\label{lem:1DiracCase}Consider an MDP with countable state and action spaces.
Let $Z, Y$ be value distributions such that each state-action distribution $Z(x,a)$, $Y(x, a)$ is given by a single Dirac. Consider the particular case where rewards are identically $0$ and $\gamma=1$, and let $\tau \in [0,1]$. Denote by $\Pi_\tau$ the projection operator that maps a probability distribution onto a Dirac delta located at its $\tau$\textsuperscript{th} quantile. Then
\begin{align*}
\overline{d}_\infty( \Pi_\tau \mathcal{T}^\pi Z, \Pi_\tau \mathcal{T}^\pi Y ) \leq \overline{d}_\infty(Z, Y)
\end{align*}
\end{lem}
\begin{proof}
Let $Z(x, a) = \delta_{\theta(x, a)}$ and $Y(x, a) = \delta_{\psi(x, a)}$ for each state-action pair $(x, a) \in \mathcal{X} \times \mathcal{A}$, for some functions $\psi, \theta : \mathcal{X} \times \mathcal{A} \rightarrow \mathbb{R}$. Let $(x^\prime, a^\prime)$ be a state-action pair, and let $((x_i, a_i))_{i \in I}$ be all the state-action pairs that are accessible from $(x^\prime, a^\prime)$ in a single transition, with $I$ a (finite or countably infinite) indexing set.
To lighten notation, we write $\theta_i$ for $\theta(x_i, a_i)$ and $\psi_i$ for $\psi(x_i, a_i)$. Further, let the probability of transitioning from $(x^\prime, a^\prime)$ to $(x_i, a_i)$ be $p_i$, for all $i\in I$.

Then we have
\begin{align}
(\mathcal{T}^\pi Z)(x^\prime, a^\prime) = \sum_{i \in I} p_i \delta_{\theta_i} \\
(\mathcal{T}^\pi Y)(x^\prime, a^\prime) = \sum_{i \in I} p_i \delta_{\psi_i} \, .
\end{align}
Now consider the $\tau$\textsuperscript{th} quantile of each of these distributions, for $\tau \in [0,1]$ arbitrary. Let $u \in I$ be such that $\theta_u$ is equal to this quantile of $(\mathcal{T}^\pi Z)(x^\prime, a^\prime)$, and let $v \in I$ such that $\psi_v$ is equal to this quantile of $(\mathcal{T}^\pi Y)(x^\prime, a^\prime)$. Now note that
\[
d_\infty( \Pi_\tau \mathcal{T}^\pi Z(x^\prime, a^\prime), \Pi_\tau \mathcal{T}^\pi Y(x^\prime, a^\prime) ) = |\theta_u - \psi_v| 
\]
We now show that 
\begin{align}\label{eq:contradictthis}
|\theta_u -\psi_v | > |\theta_i - \psi_i|\ \ \ \forall i \in I
\end{align}
is impossible, from which it will follow that
\[
d_\infty( \Pi_\tau \mathcal{T}^\pi Z(x^\prime, a^\prime), \Pi_\tau \mathcal{T}^\pi Y(x^\prime, a^\prime) )  \leq \overline{d}_\infty(Z, Y) \, ,
\]
and the result then follows by taking maxima over state-action pairs $(x^\prime, a^\prime)$.
To demonstrate the impossibility of \eqref{eq:contradictthis}, 
without loss of generality we take $\theta_u \leq \psi_v$. 

We now introduce the following partitions of the indexing set $I$. Define:
\begin{align*}
I_{\leq \theta_u} = \{ i \in I | \theta_i \leq \theta_u \} \, ,\\
I_{> \theta_u} = \{ i \in I | \theta_i > \theta_u \} \, , \\
I_{< \psi_v} = \{ i \in I | \psi_i < \psi_v \} \, , \\
I_{\geq \psi_v} = \{ i \in I | \psi_i \geq \psi_v \} \, ,
\end{align*}
and observe that we clearly have the following disjoint unions:
\begin{align*}
I = I_{\leq \theta_u} \cup I_{> \theta_u} \, , \\
I = I_{< \psi_v} \cup  I_{\geq \psi_v} \, .
\end{align*}
If \eqref{eq:contradictthis} is to hold, then we must have $I_{\leq \theta_u} \cap I_{\geq \psi_v}  = \emptyset$. Therefore, we must have $I_{\leq \theta_u} \subseteq I_{< \psi_v}$. But if this is the case, then since $\theta_u$ is the $\tau$\textsuperscript{th} quantile of $(\mathcal{T}^\pi Z)(x^\prime, a^\prime)$, we must have
\[
\sum_{i \in I_{\leq \theta_u}} p_i \geq \tau \, ,
\]
and so consequently
\[
\sum_{i \in I_{< \psi_v}} p_i \geq \tau \, ,
\]
from which we conclude that the $\tau$\textsuperscript{th} quantile of $(\mathcal{T}^\pi Y)(x^\prime, a^\prime)$ is less than $\psi_v$, a contradiction. Therefore \eqref{eq:contradictthis} cannot hold, completing the proof.
\end{proof}

\begin{lem}\label{lem:WinftyIsMaxQuantileDiff}
For any two probability distributions $\nu_1, \nu_2$ over the real numbers, and the Wasserstein projection operator $\Pi_{W_1}$ that projects distributions onto support of size $n$, we have that
\begin{align*}
& d_\infty(\Pi_{W_1} \nu_1, \Pi_{W_1} \nu_2) \\
= & \max_{i=1,\ldots,n} \left|F_{\nu_1}^{-1}\left(\frac{2i-1}{2n}\right) - F_{\nu_2}^{-1}\left(\frac{2i-1}{2n}\right)\right| \, .
\end{align*}
\end{lem}
\begin{proof}
By the discussion surrounding Lemma \ref{w1_midpoint}, we have that $\Pi_{W_1} \nu_k = \sum_{i=1}^n \frac{1}{n} \delta_{F^{-1}_{\nu_k}(\frac{2i-1}{2n})}$ for $k=1,2$. Therefore, the optimal coupling between $\Pi_{W_1} \nu_1$ and $\Pi_{W_1} \nu_2$ must be given by $F^{-1}_{\nu_1}(\frac{2i-1}{2n}) \mapsto F^{-1}_{\nu_2}(\frac{2i-1}{2n})$ for each $i=1,\ldots, n$. This immediately leads to the expression of the lemma.
\end{proof}

\subsection{Further theoretical results}

\begin{lem}\label{lem:dpnocontraction}
The projected Bellman operator $\Pi_{W_1}\mathcal{T}^\pi$ is in general not a non-expansion in $\overline{d}_p$, for $p \in [1, \infty)$.
\end{lem}
\begin{proof}
Consider the case where the number of Dirac deltas in each distribution, $N$, is equal to $2$, and let $\gamma=1$. We consider an MDP with a single initial state, $x$, and two terminal states, $x_1$ and $x_2$. We take the action space of the MDP to be trivial, and therefore omit it in the notation that follows. Let the MDP have a $2/3$ probability of transitioning from $x$ to $x_1$, and $1/3$ probability of transitioning from $x$ to $x_2$. We take all rewards in the MDP to be identically $0$. Further, consider two value distributions, $Z$ and $Y$, given by:
\begin{align*}
Z(x_1) = \frac{1}{2}\delta_0 + \frac{1}{2} \delta_2 \, , &\ \ Y(x_1) = \frac{1}{2}\delta_1 + \frac{1}{2} \delta_2 \, , \\
Z(x_2) = \frac{1}{2}\delta_3 + \frac{1}{2} \delta_5 \, , &\ \ Y(x_2) = \frac{1}{2}\delta_4 + \frac{1}{2} \delta_5 \, , \\
Z(x) = \delta_0 \, ,& \ \ Y(x) = \delta_0 \, .
\end{align*}
Then note that we have
\begin{align*}
d_p(Z(x_1), Y(x_1)) = \left(\frac{1}{2}|1 - 0|\right)^{1/p} = \frac{1}{2^{1/p}} \, , \\
d_p(Z(x_2), Y(x_2)) = \left(\frac{1}{2}|4 - 3|\right)^{1/p} = \frac{1}{2^{1/p}} \, , \\
d_p(Z(x), Y(x)) = 0 \, ,
\end{align*}
and so
\[
\overline{d}_p(Z, Y) = \frac{1}{2^{1/p}} \, .
\]
We now consider the projected backup for these two value distributions at the state $x$. We first compute the full backup:
\begin{align*}
    (\mathcal{T}^\pi Z)(x) = \frac{1}{3} \delta_0 + \frac{1}{3} \delta_2 + \frac{1}{6} \delta_3 + \frac{1}{6} \delta_5 \, ,\\
    (\mathcal{T}^\pi Y)(x) = \frac{1}{3} \delta_1 + \frac{1}{3} \delta_2 + \frac{1}{6} \delta_4 + \frac{1}{6} \delta_5 \, .
\end{align*}
Appealing to Lemma \ref{w1_midpoint}, we note that when projected these distributions onto two equally-weighted Diracs, the locations of these Diracs correspond to the 25\% and 75\% quantiles of the original distributions. We therefore have
\begin{align*}
    (\Pi_{W_1} \mathcal{T}^\pi Z)(x) = \frac{1}{2} \delta_0 + \frac{1}{2} \delta_3 \, ,\\
    (\Pi_{W_1} \mathcal{T}^\pi Y)(x) = \frac{1}{2} \delta_1 + \frac{1}{2} \delta_4 \, ,
\end{align*}
and we therefore obtain
\begin{align*}
\overline{d}_1(\Pi_{W_1} \mathcal{T}^\pi Z, \Pi_{W_1} \mathcal{T}^\pi Y) 
= & \left(\frac{1}{2}(|1-0|^p + |4-3|^p) \right)^{1/p} \\
= & 1 > \frac{1}{2^{1/p}} =  \overline{d}_1(Z, Y) \, ,
\end{align*}
completing the proof.
\end{proof}

\subsection{Notation}
Human-normalized scores are given by \cite{vanhasselt16deep},
\begin{equation}
    \nonumber score = \frac{agent - random}{human - random},
\end{equation}
where $agent$, $human$ and $random$ represent the per-game raw scores for the agent, human baseline, and random agent baseline.

\begin{table}[ht]
\centering
\caption{Notation used in the paper}
\label{my-label}
\begin{tabular}{lll}
  Symbol        &  Description of usage\\
  \hline \\
                & Reinforcement Learning \\
  \hline
  $\cM$         &  MDP ($\cX$, $\cA$, $R$, $P$, $\gamma$) \\
  $\cX$         &  State space of MDP \\
  $\cA$         &  Action space of MDP \\
  $R$, $R_t$    &  Reward function, random variable reward \\
  $P$           &  Transition probabilities, $P(x' | x, a)$ \\
  $\gamma$      &  Discount factor, $\gamma \in [0, 1)$ \\
  $x,x_t \in \cX$ & States \\
  $a, a^*, b \in \cA$ & Actions \\
  $r, r_t \in R$    &   Rewards\\
  $\pi$         & Policy \\
  $\cT^\pi$     & (dist.) Bellman operator \\
  $\cT$         & (dist.) Bellman optimality operator \\
  $V^\pi$, $V$  & Value function, state-value function \\
  $Q^\pi$, $Q$  & Action-value function \\
  $\alpha$      & Step-size parameter, learning rate \\
  $\epsilon$    & Exploration rate, $\epsilon$-greedy \\
  $\epsilon_{ADAM}$ & Adam parameter \\
  $\kappa$      & Huber-loss parameter  \\
  $\cL_\kappa$  & Huber-loss with parameter $\kappa$ \\
  \hline \\
                & Distributional Reinforcement Learning \\
  \hline
  $Z^\pi$, $Z$  & Random return, value distribution \\
  $Z^\pi_{MC}$  & Monte-Carlo value distribution under policy $\pi$ \\
  $\cZ$         & Space of value distributions \\
  $\hat{Z}^\pi$ & Fixed point of convergence for $\Pi_{W_1} \cT^\pi$ \\
  $z \sim Z$    & Instantiated return sample \\
  $p$           & Metric order \\
  $W_p$         & $p$-Wasserstein metric \\
  $L^p$         & Metric order $p$ \\
  $\dip$        & maximal form of Wasserstein \\
  $\Phi$        & Projection used by $\cfo$\\
  $\Pi_{W_1}$   & $1$-Wasserstein projection \\
  $\rho_\tau$   & Quantile regression loss \\
  $\rho_\tau^\kappa$ & Huber quantile loss \\
  $q_1,\ldots,q_N$ & Probabilities, parameterized probabilities \\
  $\tau_0, \tau_1, \ldots, \tau_N$ & Cumulative probabilities with $\tau_0 := 0$\\
  $\hat{\tau}_1, \ldots, \hat{\tau}_N$ & Midpoint quantile targets\\
  $\omega$      & Sample from unit interval \\
  $\delta_z$    & Dirac function at $z \in \mathbb{R}$\\
  $\theta$      & Parameterized function \\
  $B$           & Bernoulli distribution \\
  $B_\mu$     & Parameterized Bernoulli distribution \\
  $\qcZ$        & Space of quantile (value) distributions \\
  $Z_\theta$    & Parameterized quantile (value) distribution \\
  $Y$           & Random variable over $\mathbb{R}$\\
  $Y_1, \ldots, Y_m$ & Random variable samples\\
  $\hat{Y}_m$   & Empirical distribution from $m$-Diracs
\end{tabular}
\end{table}

\begin{figure*}[t]
\begin{center}
\includegraphics[width=\textwidth]{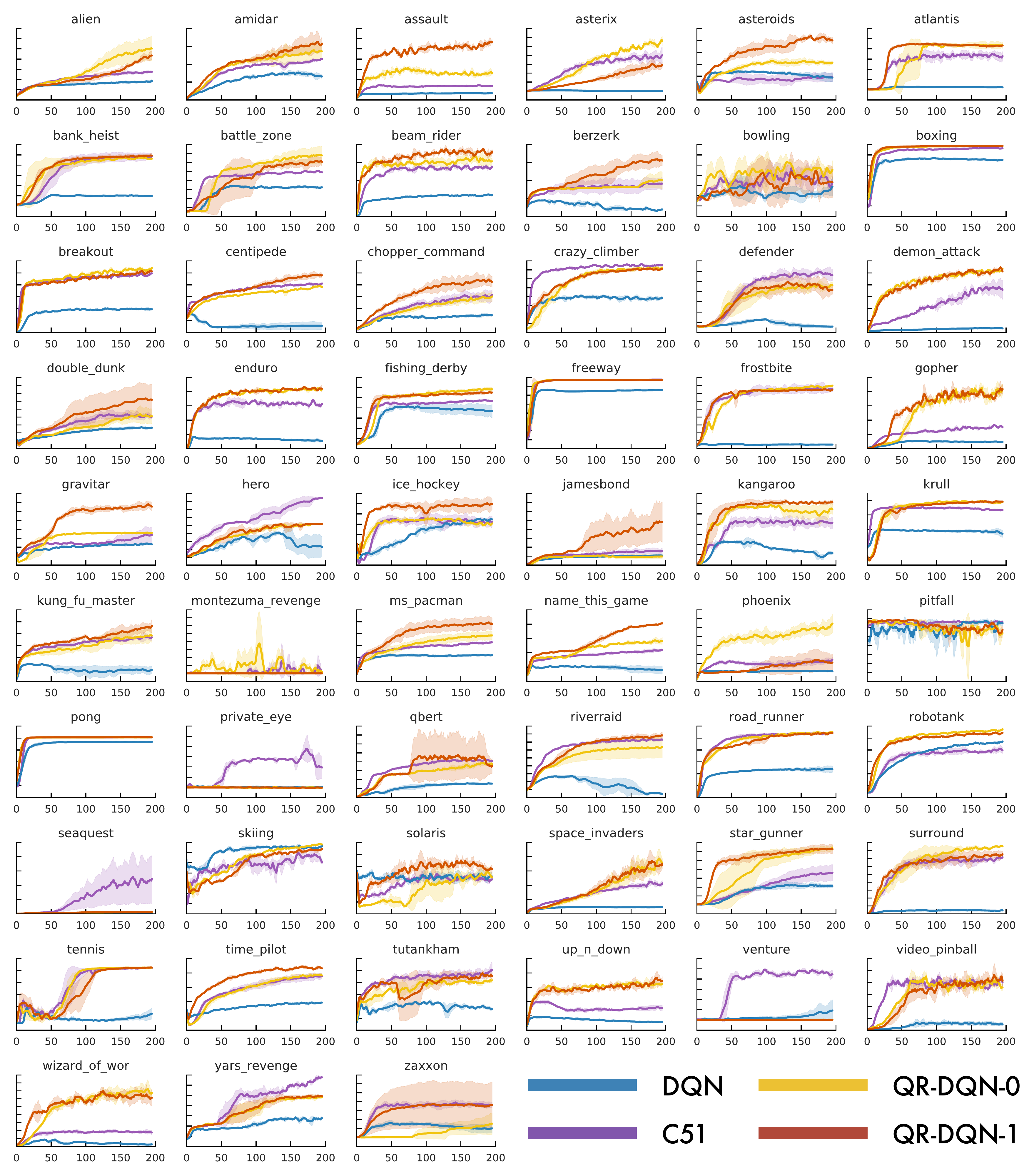}
\end{center}
\caption{Online training curves for $\dqn$, $\cfo$, and $\qrdqn$ on 57 Atari 2600 games. Curves are averages over three seeds, smoothed over a sliding window of 5 iterations, and error bands give standard deviations.\label{fig:all_games}}
\end{figure*}

\newpage

\begin{figure*}
\small
\centering
\begin{tabular}{ l | r|r|r|r|r|r| r }
  \textbf{\textsc{games}}  &  \textbf{\textsc{random}}  &  \textbf{\textsc{human}}  &  \textbf{\textsc{dqn}}  &   \textbf{\textsc{prior.}} \textbf{\textsc{duel.}}  &  \textbf{\textsc{c51}} & \textbf{\textsc{qr-dqn-0}} & \textbf{\textsc{qr-dqn-1}}\\
\hline
Alien & 227.8 & 7,127.7 & 1,620.0 & 3,941.0 & 3,166 & \textbf{\textcolor{blue}{9,983}} & 4,871 \\
Amidar & 5.8 & 1,719.5 & 978.0 & 2,296.8 & 1,735 & \textbf{\textcolor{blue}{2,726}} & 1,641 \\
Assault & 222.4 & 742.0 & 4,280.4 & 11,477.0 & 7,203 & 19,961 & \textbf{\textcolor{blue}{22,012}} \\
Asterix & 210.0 & 8,503.3 & 4,359.0 & 375,080.0 & 406,211 & \textbf{\textcolor{blue}{454,461}} & 261,025 \\
Asteroids & 719.1 & \textbf{\textcolor{blue}{47,388.7}} & 1,364.5 & 1,192.7 & 1,516 & 2,335 & 4,226 \\
Atlantis & 12,850.0 & 29,028.1 & 279,987.0 & 395,762.0 & 841,075 & \textbf{\textcolor{blue}{1,046,625}} & 971,850 \\
Bank Heist & 14.2 & 753.1 & 455.0 & \textbf{\textcolor{blue}{1,503.1}} & 976 & 1,245 & 1,249 \\
Battle Zone & 2,360.0 & 37,187.5 & 29,900.0 & 35,520.0 & 28,742 & 35,580 & \textbf{\textcolor{blue}{39,268}} \\
Beam Rider & 363.9 & 16,926.5 & 8,627.5 & 30,276.5 & 14,074 & 24,919 & \textbf{\textcolor{blue}{34,821}} \\
Berzerk & 123.7 & 2,630.4 & 585.6 & 3,409.0 & 1,645 & \textbf{\textcolor{blue}{34,798}} & 3,117 \\
Bowling & 23.1 & \textbf{\textcolor{blue}{160.7}} & 50.4 & 46.7 & 81.8 & 85.3 & 77.2 \\
Boxing & 0.1 & 12.1 & 88.0 & 98.9 & 97.8 & 99.8 & \textbf{\textcolor{blue}{99.9}} \\
Breakout & 1.7 & 30.5 & 385.5 & 366.0 & 748 & \textbf{\textcolor{blue}{766}} & 742 \\
Centipede & 2,090.9 & 12,017.0 & 4,657.7 & 7,687.5 & 9,646 & 9,163 & \textbf{\textcolor{blue}{12,447}} \\
Chopper Command & 811.0 & 7,387.8 & 6,126.0 & 13,185.0 & \textbf{\textcolor{blue}{15,600}} & 7,138 & 14,667 \\
Crazy Climber & 10,780.5 & 35,829.4 & 110,763.0 & 162,224.0 & 179,877 & \textbf{\textcolor{blue}{181,233}} & 161,196 \\
Defender & 2,874.5 & 18,688.9 & 23,633.0 & 41,324.5 & 47,092 & 42,120 & \textbf{\textcolor{blue}{47,887}} \\
Demon Attack & 152.1 & 1,971.0 & 12,149.4 & 72,878.6 & \textbf{\textcolor{blue}{130,955}} & 117,577 & 121,551 \\
Double Dunk & -18.6 & -16.4 & -6.6 & -12.5 & 2.5 & 12.3 & \textbf{\textcolor{blue}{21.9}} \\
Enduro & 0.0 & 860.5 & 729.0 & 2,306.4 & \textbf{\textcolor{blue}{3,454}} & 2,357 & 2,355 \\
Fishing Derby & -91.7 & -38.7 & -4.9 & \textbf{\textcolor{blue}{41.3}} & 8.9 & 37.4 & 39.0 \\
Freeway & 0.0 & 29.6 & 30.8 & 33.0 & 33.9 & \textbf{\textcolor{blue}{34.0}} & \textbf{\textcolor{blue}{34.0}} \\
Frostbite & 65.2 & 4,334.7 & 797.4 & \textbf{\textcolor{blue}{7,413.0}} & 3,965 & 4,839 & 4,384 \\
Gopher & 257.6 & 2,412.5 & 8,777.4 & 104,368.2 & 33,641 & \textbf{\textcolor{blue}{118,050}} & 113,585 \\
Gravitar & 173.0 & \textbf{\textcolor{blue}{3,351.4}} & 473.0 & 238.0 & 440 & 546 & 995 \\
H.E.R.O. & 1,027.0 & 30,826.4 & 20,437.8 & 21,036.5 & \textbf{\textcolor{blue}{38,874}} & 21,785 & 21,395 \\
Ice Hockey & -11.2 & \textbf{\textcolor{blue}{0.9}} & -1.9 & -0.4 & -3.5 & -3.6 & -1.7 \\
James Bond & 29.0 & 302.8 & 768.5 & 812.0 & 1,909 & 1,028 & \textbf{\textcolor{blue}{4,703}} \\
Kangaroo & 52.0 & 3,035.0 & 7,259.0 & 1,792.0 & 12,853 & 14,780 & \textbf{\textcolor{blue}{15,356}} \\
Krull & 1,598.0 & 2,665.5 & 8,422.3 & 10,374.4 & 9,735 & 11,139 & \textbf{\textcolor{blue}{11,447}} \\
Kung-Fu Master & 258.5 & 22,736.3 & 26,059.0 & 48,375.0 & 48,192 & 71,514 & \textbf{\textcolor{blue}{76,642}} \\
Montezuma's Revenge & 0.0 & \textbf{\textcolor{blue}{4,753.3}} & 0.0 & 0.0 & 0.0 & 75.0 & 0.0 \\
Ms. Pac-Man & 307.3 & \textbf{\textcolor{blue}{6,951.6}} & 3,085.6 & 3,327.3 & 3,415 & 5,822 & 5,821 \\
Name This Game & 2,292.3 & 8,049.0 & 8,207.8 & 15,572.5 & 12,542 & 17,557 & \textbf{\textcolor{blue}{21,890}} \\
Phoenix & 761.4 & 7,242.6 & 8,485.2 & \textbf{\textcolor{blue}{70,324.3}} & 17,490 & 65,767 & 16,585 \\
Pitfall! & -229.4 & \textbf{\textcolor{blue}{6,463.7}} & -286.1 & 0.0 & 0.0 & 0.0 & 0.0 \\
Pong & -20.7 & 14.6 & 19.5 & 20.9 & 20.9 & \textbf{\textcolor{blue}{21.0}} & \textbf{\textcolor{blue}{21.0}} \\
Private Eye & 24.9 & \textbf{\textcolor{blue}{69,571.3}} & 146.7 & 206.0 & 15,095 & 146 & 350 \\
Q*Bert & 163.9 & 13,455.0 & 13,117.3 & 18,760.3 & 23,784 & 26,646 & \textbf{\textcolor{blue}{572,510}} \\
River Raid & 1,338.5 & 17,118.0 & 7,377.6 & \textbf{\textcolor{blue}{20,607.6}} & 17,322 & 9,336 & 17,571 \\
Road Runner & 11.5 & 7,845.0 & 39,544.0 & 62,151.0 & 55,839 & \textbf{\textcolor{blue}{67,780}} & 64,262 \\
Robotank & 2.2 & 11.9 & \textbf{\textcolor{blue}{63.9}} & 27.5 & 52.3 & 61.1 & 59.4 \\
Seaquest & 68.4 & 42,054.7 & 5,860.6 & 931.6 & \textbf{\textcolor{blue}{266,434}} & 2,680 & 8,268 \\
Skiing & -17,098.1 & \textbf{\textcolor{blue}{-4,336.9}} & -13,062.3 & -19,949.9 & -13,901 & -9,163 & -9,324 \\
Solaris & 1,236.3 & \textbf{\textcolor{blue}{12,326.7}} & 3,482.8 & 133.4 & 8,342 & 2,522 & 6,740 \\
Space Invaders & 148.0 & 1,668.7 & 1,692.3 & 15,311.5 & 5,747 & \textbf{\textcolor{blue}{21,039}} & 20,972 \\
Star Gunner & 664.0 & 10,250.0 & 54,282.0 & \textbf{\textcolor{blue}{125,117.0}} & 49,095 & 70,055 & 77,495 \\
Surround & -10.0 & 6.5 & -5.6 & 1.2 & 6.8 & \textbf{\textcolor{blue}{9.7}} & 8.2 \\
Tennis & -23.8 & -8.3 & 12.2 & 0.0 & 23.1 & \textbf{\textcolor{blue}{23.7}} & 23.6 \\
Time Pilot & 3,568.0 & 5,229.2 & 4,870.0 & 7,553.0 & 8,329 & 9,344 & \textbf{\textcolor{blue}{10,345}} \\
Tutankham & 11.4 & 167.6 & 68.1 & 245.9 & 280 & \textbf{\textcolor{blue}{312}} & 297 \\
Up and Down & 533.4 & 11,693.2 & 9,989.9 & 33,879.1 & 15,612 & 53,585 & \textbf{\textcolor{blue}{71,260}} \\
Venture & 0.0 & 1,187.5 & 163.0 & 48.0 & \textbf{\textcolor{blue}{1,520}} & 0.0 & 43.9 \\
Video Pinball & 16,256.9 & 17,667.9 & 196,760.4 & 479,197.0 & \textbf{\textcolor{blue}{949,604}} & 701,779 & 705,662 \\
Wizard Of Wor & 563.5 & 4,756.5 & 2,704.0 & 12,352.0 & 9,300 & \textbf{\textcolor{blue}{26,844}} & 25,061 \\
Yars' Revenge & 3,092.9 & 54,576.9 & 18,098.9 & \textbf{\textcolor{blue}{69,618.1}} & 35,050 & 32,605 & 26,447 \\
Zaxxon & 32.5 & 9,173.3 & 5,363.0 & \textbf{\textcolor{blue}{13,886.0}} & 10,513 & 7,200 & 13,112
\end{tabular}
\caption{Raw scores across all games, starting with 30 no-op actions. Reference values from \citet{wang2016dueling} and \citet{c51}.\label{fig:atari_sota}}
\end{figure*}

\end{document}